\documentclass[sn-mathphys]{sn-jnl}

\usepackage{xurl}

\usepackage{xfrac}

\usepackage{graphicx}
\graphicspath{{graphics/}}
\usepackage{amsmath}
\usepackage{amssymb}
\usepackage{thmtools}
\usepackage{xcolor}
\usepackage{dsfont}
\usepackage{multirow}
\usepackage{ulem} 

\usepackage{enumitem}

\usepackage{subcaption}
\usepackage{array}
\newcommand{\PreserveBackslash}[1]{\let\temp=\\#1\let\\=\temp}
\newcolumntype{C}[1]{>{\PreserveBackslash\centering}p{#1}}
\newcolumntype{R}[1]{>{\PreserveBackslash\raggedleft}p{#1}}
\newcolumntype{L}[1]{>{\PreserveBackslash\raggedright}p{#1}}

\theoremstyle{thmstyleone}
\newtheorem{theorem}{Theorem}
\theoremstyle{thmstyleone}
\newtheorem{lemma}{Lemma}
\theoremstyle{thmstyleone}
\newtheorem{corollary}{Corollary}

\newcommand{\refeq}[1]{\eqref{eq:#1}}
\newcommand{\reffig}[1]{Fig.~\ref{fig:#1}}

\newcommand{\refthm}[1]{Theorem~\ref{thm:#1}}
\newcommand{\refsec}[1]{Section~\ref{sec:#1}}
\newcommand{\refapp}[1]{Appendix~\ref{app:#1}}

\newcommand{\mycase}[1]{\text{if $#1$}}
\newcommand{\otherwise}{\text{otherwise}}
\newenvironment{multipartdef}
  {\begin{cases}}
  {\end{cases}}

\newcommand{\di}{\delta}
\newcommand{\er}{\epsilon}
\newcommand{\cl}{\phi}
\newcommand{\op}{\gamma}

\newcommand{\E}[1]{\mathbb{E}\left[#1\right]}
\newcommand{\Dir}{\mathrm{Dir}}

\newcommand{\ps}[1]{\mathds{P}_{#1}}
\newcommand{\D}{\mathbb{D}}

\newcommand{\subeq}[1]{(\mathrm{#1})}

\jyear{2022}

\begin{document}

\title[Morphology on categorical distributions]{Morphology on categorical distributions}

\author*[1,2]{\fnm{Silas Nyboe} \sur{Ørting}}\email{silas@di.ku.dk}
\author[1,2]{\fnm{Hans Jacob Teglbjærg} \sur{Stephensen}}
\author[1,2]{\fnm{Jon} \sur{Sporring}}
\affil[1]{\orgdiv{Department of Computer Science} \orgname{University of Copenhagen}}
\affil[2]{\orgname{QIM - Center for Quantification of Imaging Data from MAX IV}}

\abstract{Mathematical morphology (MM) is an indispensable tool for post-processing. Several extensions of MM to categorical images, such as multi-class segmentations, have been proposed. However, none provide satisfactory definitions for morphology on probabilistic representations of categorical images.

  The categorical distribution is a natural choice for representing uncertainty about categorical images. Extending MM to categorical distributions is problematic because categories are inherently unordered. Without ranking categories, we cannot use the standard framework based on supremum and infimum. Ranking categories is impractical and problematic. Instead, we consider the probabilistic representation and operations that emphasize a single category.
  
  In this work, we review and compare previous approaches. We propose two approaches for morphology on categorical distributions: operating on Dirichlet distributions over the parameters of the distributions, and operating directly on the distributions. We propose a ``protected'' variant of the latter and demonstrate the proposed approaches by fixing misclassifications and modeling annotator bias.}

\maketitle

\section{Introduction}
Multi-class segmentation problems are common in analysis of biomedical images. A typical solution is to train a neural network pixel classifier. Commonly, these networks predict a probability distribution over all classes in each pixel, which can be thresholded to obtain a final segmentation. These predictions often contains holes, partial misclassifications, shrinkage of small classes and rough borders between classes, resulting in errors in the final segmentation. To improve the segmentation, post-processing is often used to close holes, reclassify uncertain pixel labels based on proximity, grow objects, and smoothen rough boundaries. 

Mathematical morphology is a powerful framework for post-processing binary and grayscale images. Binary and grayscale morphology are special cases of morphology on complete lattices \cite{serra1994morphological}. A complete lattice is a partially ordered set (poset), where each non-empty subset has an infimum and a supremum. For complete lattices the core operators, dilation and erosion, can be defined using supremum and infimum: for binary morphology using set union and intersection; for grayscale morphology using maximum and minimum under the standard total ordering of the reals, see \cite{serra1994morphological} for an in depth treatment of the theoretical foundations of mathematical morphology.

For general multi-class images, there is no natural ordering of the classes, and hence, they do not form a complete lattice.
For example, for a segmentation of microscope images of cells into cell membrane, mitochondria and background, any ordering of the classes is task dependent and not given by the images themselves. 
A natural representation of this kind of data is the categorical distribution, which can represent both crisp segmentation masks and uncertainty as encountered in prediction images. In the remainder of this work we will use the term ``categorical'' instead of ``multi-class''.

In this work, we provide a thorough review of previously proposed approaches to morphology on categorical images. We then propose two approaches for morphology on categorical distributions, an indirect approach where we operate on Dirichlet distributions that are then transformed to categorical distributions, and a direct approach where we operate on the categorical distributions themselves. We then define protected variants of the direct operations that allow finer control over the processing. Finally, we illustrate the utility of the proposed approach on two tasks: fixing misclassified mitochondria and modeling annotator bias.

\section{Background and related work}
In this section we briefly restate morphology on complete lattices and on binary and grayscale images, before we review the most relevant literature \cite{busch1995morphological,koppen2000pareto,hanbury2001morphological,ronse2005morphology,chevallier2016nary,vandegronde2017nonscalar,grossiord2019shape}. What we refer to as categorical images have various names in the literature: color-coded images, label images and n-ary images. In the sections below we will use the original names in the section titles, but otherwise we will refer to categorical images and categorical morphology.

In the literature, there are three main approaches for extending morphology to images with values that do not have a natural ordering: impose an order on the values, which is the common approach for color images;  operate on all categories simultaneously \cite{busch1995morphological,ronse2005morphology}; and operate on a single category at a time \cite{chevallier2016nary,vandegronde2017nonscalar}

Morphology on color images has received a lot of attention, with the primary focus on ordering colors by exploiting the relationship between dimensions of color spaces. See for example \cite{aptoula2007acomparative} for an overview of approaches for defining an ordering of colors. Our focus is on categorical images, where such approaches are less relevant.

\subsection{Morphology on complete lattices}
\label{sec:complete-lattice}
Let $\Gamma$ be a set with the partial order $\le$. The poset $(\Gamma, \le)$ is a complete lattice if every subset of $\Gamma$ has an infimum $\wedge$ and a supremum $\vee$.
We define an image as a function $f$ from pixel-coordinates $\D = \mathbb{Z}^d$ to $\Gamma$, and a structuring element $B$ as a subset of $\D$
\begin{align}
  \label{eq:image}
  f &\in \mathcal{F} = \left\{ g \mid g : \D \mapsto \Gamma \right\},  \\
  \label{eq:structuring-element}
  B &\subseteq \D.
\end{align}
The dilation $(\di)$ and erosion $(\er)$ of $f$ by $B$ are then defined as the supremum and infimum over the local neighborhoods in $f$ given by $B$
\begin{align}
  \label{eq:dilation-lattice}
  \di(f;B)(x) &= \bigvee\limits_{\{y\mid (y-x) \in B\}} f(y),\\
  \label{eq:erosion-lattice}
  \er(f;B)(x) &= \bigwedge\limits_{\{y\mid (y-x) \in B\}} f(y).
\end{align}
Opening ($\op$) and closing ($\cl$) are the compositions of dilation and erosion
\begin{align}
  \label{eq:opening}
  \op(f;B)(x) &= \di(\er(f;B);B),\\
  \label{eq:closing}
  \cl(f;B)(x) &= \er(\di(f;B);B).
\end{align}

\subsection{Binary and grayscale morphology}
\label{sec:binary-grayscale}
We define a grayscale image as in \refeq{image} with $\Gamma = [0,1]$. Let $\le$ be the usual ordering of the reals, then the poset $([0,1], \le)$ is a complete lattice, where the $\min$ function gives the infimum and the $\max$ function the supremum. Let $B$ be defined as in \refeq{structuring-element}. Dilation and erosion can then be obtained from \refeq{dilation-lattice} and \refeq{erosion-lattice} as 
\begin{align}
  \label{eq:grayscale-dilation}
  \di(f;B)(x) &= \max\limits_{\{y\mid (y-x) \in B\}} f(y),\\
  \label{eq:grayscale-erosion}
  \er(f;B)(x) &= \min\limits_{\{y\mid (y-x) \in B\}} f(y).
\end{align}
If we restrict $\Gamma$ to $\{0,1\}$ we obtain binary morphology.

\subsection{Morphology on color-coded images}
\label{sec:color-coded}
In~\cite{busch1995morphological} the authors propose a framework for categorical morphology where pixels have a set of categories. Let $C = \{c_1, c_2, \dots, c_n\}$ be a set of $n$ categories. The powerset of $C$, $\ps{C}$, is the set of all subsets of $C$, including the empty set. An image $f$ is then defined as in \refeq{image} with $\Gamma = \ps{C}$.
In this framework the value of a pixel can be any element of $\ps{C}$, e.g $\{c_1\}$, $\{c_1, c_n\}$ or $\{\}$. Let $\subseteq$ be the usual subset relation, then the poset $(\ps{C}, \subseteq)$ is a complete lattice where set intersection is the infimum and set union is the supremum. In~\cite{busch1995morphological} the authors propose to use structuring elements of the same form as $f$, that is $B \in \mathcal{F}$. For the sake of comparison, we first consider the simpler case where $B$ be is defined as in \refeq{structuring-element}. Dilation and erosion can then be obtained from \refeq{dilation-lattice} and \refeq{erosion-lattice} as
\begin{align}
  \label{eq:set-based-dilation}
  \di(f;B)(x) &= \bigcup\limits_{\{y\mid (y-x) \in B\}} f(y),\\
  \label{eq:set-based-erosion}
  \er(f;B)(x)  &= \bigcap\limits_{\{y\mid (y-x) \in B\}} f(y).
\end{align}
An example of these operations is shown \reffig{color-coded-set-op}. 

Let $B \in \mathcal{F}$. Under this scheme, an operation is only performed when one or more categories in the structuring element match a category in the image, and the result depends on the categories in both image and structuring element. Several variations of dilation and erosion are proposed in \cite{busch1995morphological}, here we only consider the ``transparent'' operations.
Let $\D_f$ be the domain of $f$ and $\D_B$ the domain of $B$.  A specified reference point, $y_0 \in \D_B$, is used to determine if $B$ matches $f$ and could for example be the center of a ball shaped $B$. Dilation and erosion are then defined as 
 \begin{align}
   \label{eq:color-coded-dilation}
   \di(f;B)(x) &= f(x) \cup \bigcup\limits_{\{y \in \D_B \mid f(x+y) \cap B(y_0) \ne \emptyset\}} B(y)\\
   \label{eq:color-coded-erosion}
   \er(f;B)(x) &=
                \begin{multipartdef}
                  f(x), & \mycase{f(x) \cap B(y_0) = \emptyset}\\
                  f(x)\setminus B(y_0), & \mycase{[\exists y \in \D_B](f(x+y) \cap B(y_0) = \emptyset)},\\
                  f(x) \cup B(y_0), & \otherwise
                \end{multipartdef}
 \end{align}
An example of these operations is shown \reffig{color-coded} using a cross-shaped structuring element with $y_0$ in the center.

\subsection{Morphology on label images}
\label{sec:label-images}
In~\cite{ronse2005morphology} the authors propose a framework for categorical morphology where pixels have no category ($\bot$), a unique category or conflicting categories ($\top$). Let $C = \{c_1, c_2, \dots, c_n\}$ be a set of $n$ categories and let $C_* = C \cup \{\bot, \top\}$. An image $f$ is then defined as in \refeq{image} with $\Gamma = C_*$.
The poset ($C_*, \le$), where $\le$ satisfies $[\forall c \in C](\bot \le c \le \top)$ is a complete lattice. Let $B$ be defined as in \refeq{structuring-element} and let $V(x) = \{f(x-y) \mid y \in B\}$. Dilation and erosion are then defined as
\begin{align}
  \label{eq:label-dilation}
   \di(f;B)(x) &= \begin{multipartdef}
     \top, & \mycase{\top \in V(x)}\\
     \top, & \mycase{\vert V(x) \cap C\vert  > 1}\\
     V(x) \cap C, & \mycase{\vert V(x) \cap C \vert  = 1}\\
     \bot, & \otherwise 
   \end{multipartdef}\\
   \label{eq:label-erosion}
  \er(f;B)(x) &= \begin{multipartdef}
    \bot, & \mycase{\bot \in V(x)}\\
    \bot, & \mycase{\vert V(x) \cap C\vert  > 1}\\
    V(x) \cap C, & \mycase{\vert V(x) \cap C\vert  = 1}\\
    \top, & \otherwise 
   \end{multipartdef}
\end{align}
An example of these operations is shown \reffig{label-images}. In the context of categorical distributions, where we have detailed information about label uncertainty, this approach is unsuitable due to the loss of information.

\subsection{N-ary morphology}
\label{sec:n-ary}
In~\cite{chevallier2016nary} the authors propose a framework for categorical morphology where pixels have a unique category. Let $C = \{c_1,c_2,\dots,c_n\}$ be a set of $n$ categories. An image $f$ is then defined as in \refeq{image} with $\Gamma = C$. Instead of operating on all categories simultaneously, the authors propose to operate on a single category at a time. Let $B$ be defined as in \refeq{structuring-element} and let $i$ be the category we operate on. We use subscripts to distinguish single category operations from standard operations. Dilation and erosion are then defined as 
\begin{align}
  \label{eq:n-ary-dilation}
  \di_i(f;B)(x) &= \begin{multipartdef}
    f(x), & \mycase{[\forall y \in B](f(x+y) \ne i)}\\
    i,    & \otherwise
  \end{multipartdef}\\
  \label{eq:n-ary-erosion}
  \er_i(f;B)(x) &= \begin{multipartdef}
    f(x),        & \mycase{f(x) \ne i}\\
    i,           & \mycase{[\forall y \in B](f(x+y) = i)}\\
    \theta(x,f), & \otherwise
  \end{multipartdef}
\end{align}
where $\theta$ is a function that assigns a value in the case where there are different categories in the neighborhood of $x$. A natural choice for $\theta$, which is also suggested in \cite{chevallier2016nary}, is to pick the value of the closest pixels. However, this does not help when the closest pixels have different values, which is a fundamental problem when pixel values cannot represent uncertainty. This is solved by ranking the categories and using the ranking to break ties. In general there is no obvious way of ranking categories based on the image alone, and as the number of multi-category interfaces increases it becomes more difficult to understand how one particular ranking influence the outcome.

Without ranking categories a priori, the above definition implies an ordering $\le_i$, which is not a partial order and thus $(C, \le_i)$ is not a complete lattice. In~\cite{vandegronde2017nonscalar} the authors show that $\le_i$ is a preorder, and formalize constraints for choosing $\theta$ such that dilation and erosion form an adjunction and their compositions are an opening and a closing. However, this does not help decide which category to choose when multiple categories are closest, as the constraints on $\theta$ do not yield a unique rule for breaking ties.
An example of these operations is shown in \reffig{n-ary}, where the question marks highlight two pixels that cannot be assigned a value without a method for breaking ties.

\begin{figure}[p]
  \begin{subfigure}{\textwidth}
    \includegraphics[width=0.29\textwidth]{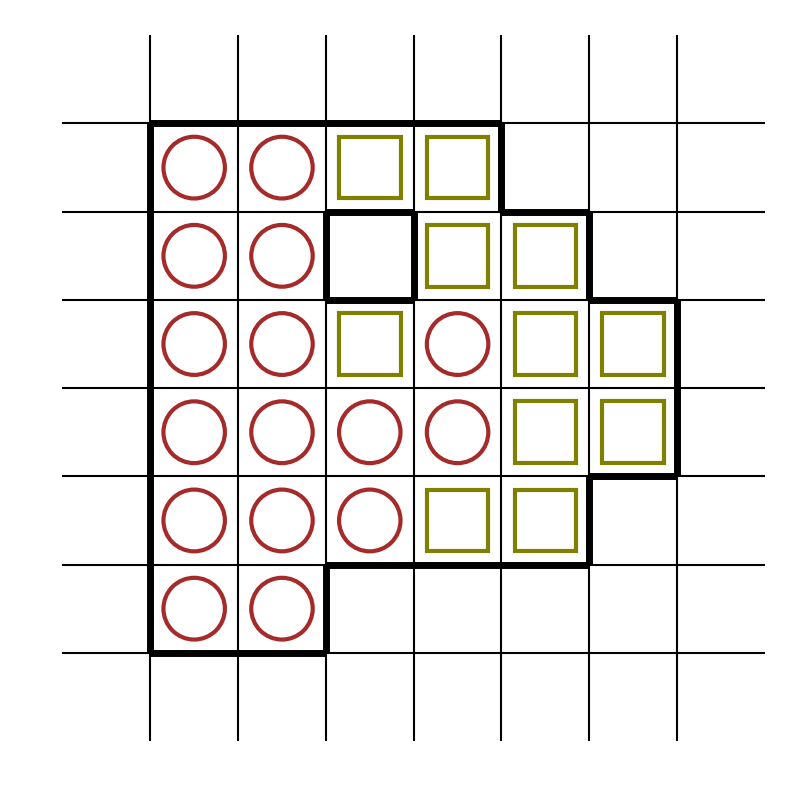}
    \includegraphics[width=0.10\textwidth]{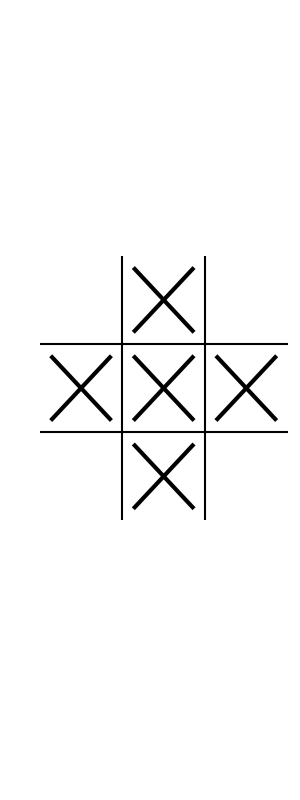}
    \includegraphics[width=0.29\textwidth]{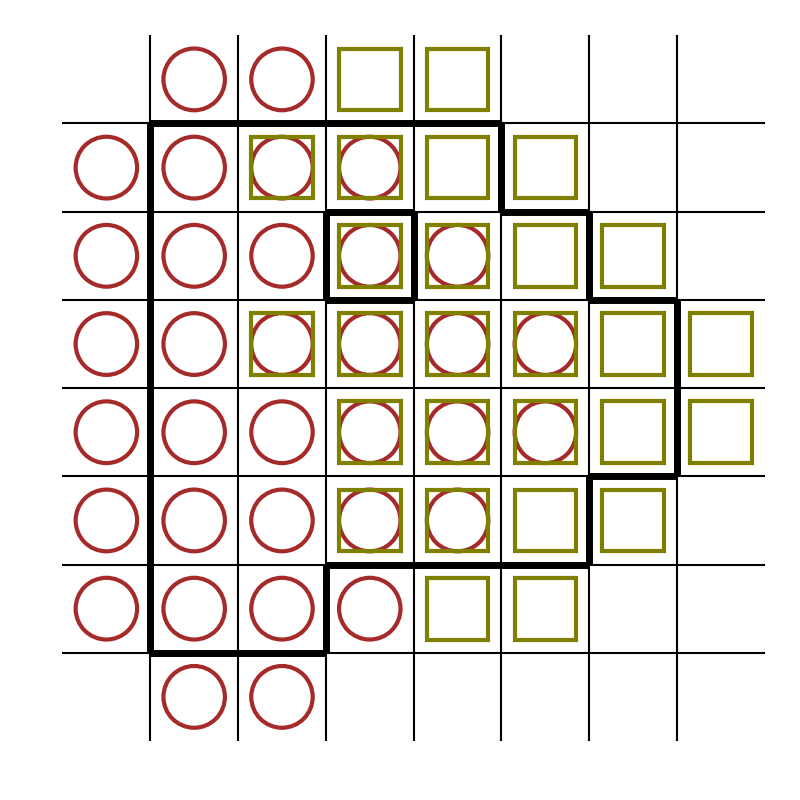}
    \includegraphics[width=0.29\textwidth]{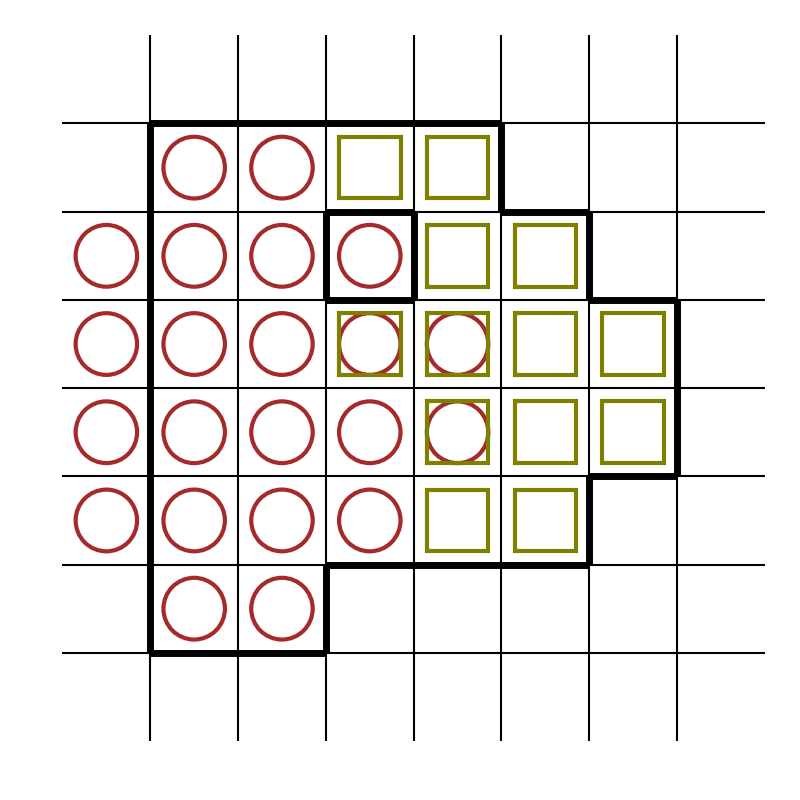}
    \caption{Morphology on color-coded images using \refeq{structuring-element} as structuring element.
     The bold boundary highlights the pixels with one or more categories and is only for illustration.
   }
   \label{fig:color-coded-set-op}
 \end{subfigure}
 
 \begin{subfigure}{\textwidth}
   \includegraphics[width=0.29\textwidth]{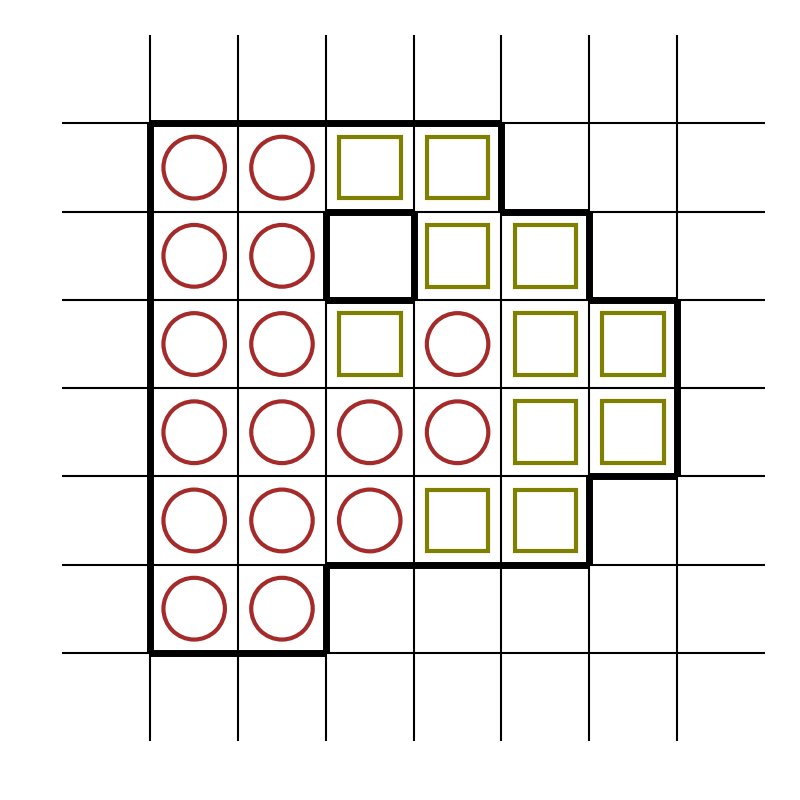}
   \includegraphics[width=0.10\textwidth]{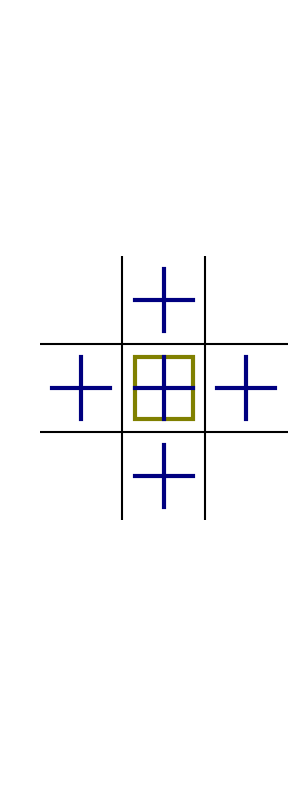}
   \includegraphics[width=0.29\textwidth]{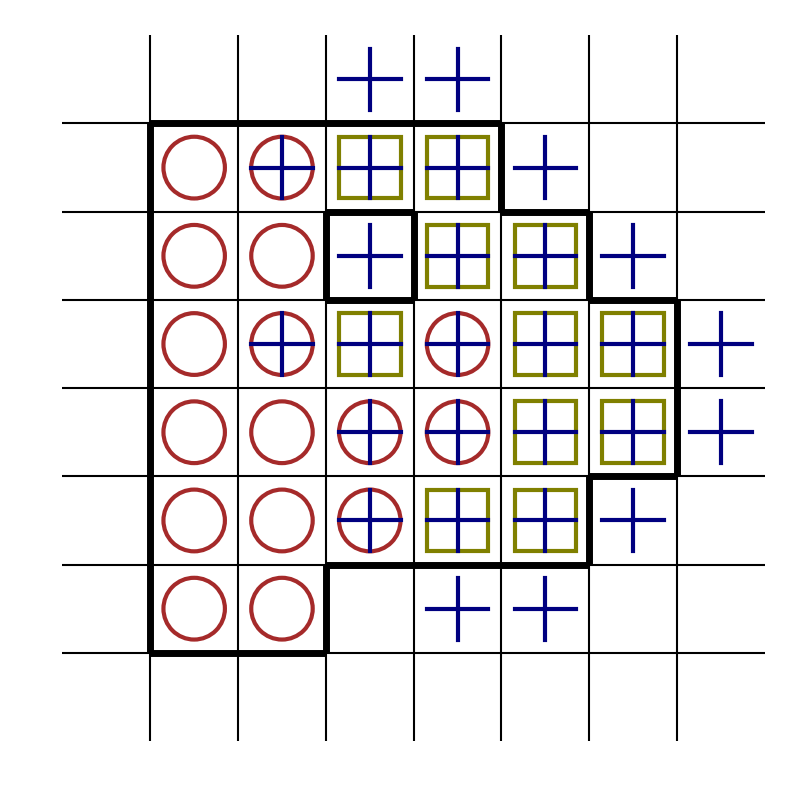}
   \includegraphics[width=0.29\textwidth]{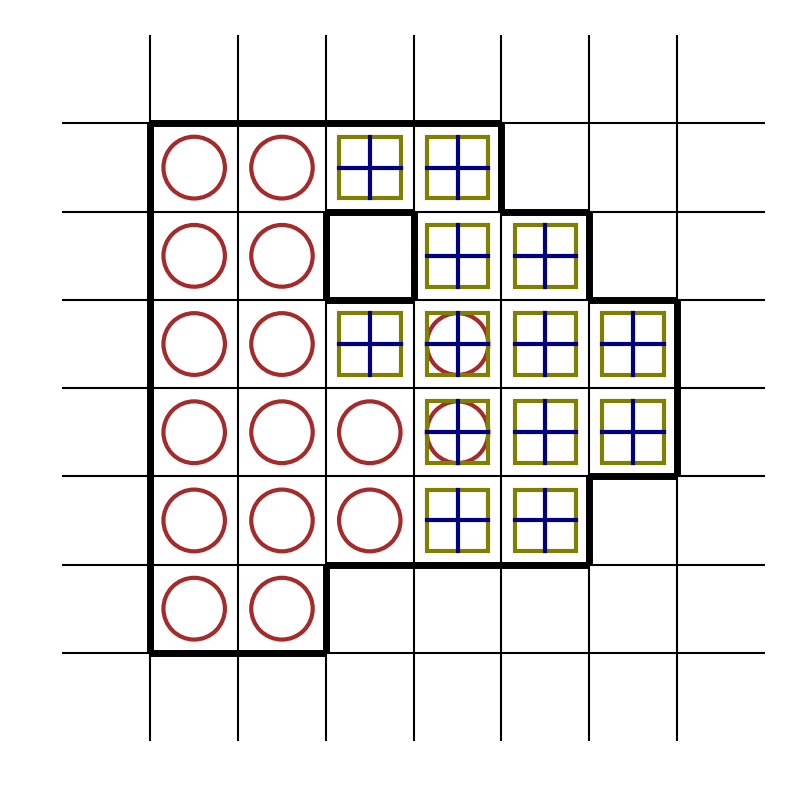}
   \caption{Morphology on color coded images using structuring element from \cite{busch1995morphological} . This figure extends Figure 2. in \cite{busch1995morphological} with the closing operation.
     Notice that $B(y_0) = \{plus, square\}$, meaning that $B$ will match any pixel with either $plus$ or $square$.
   }
   \label{fig:color-coded}
 \end{subfigure}
 
  \begin{subfigure}{\textwidth}
    \includegraphics[width=0.29\textwidth]{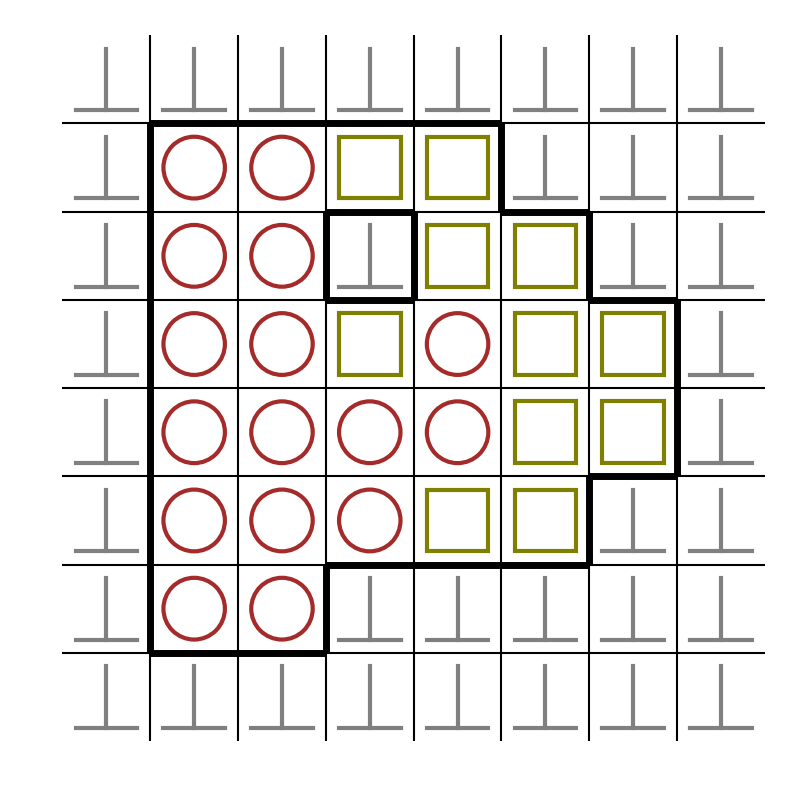}
    \includegraphics[width=0.10\textwidth]{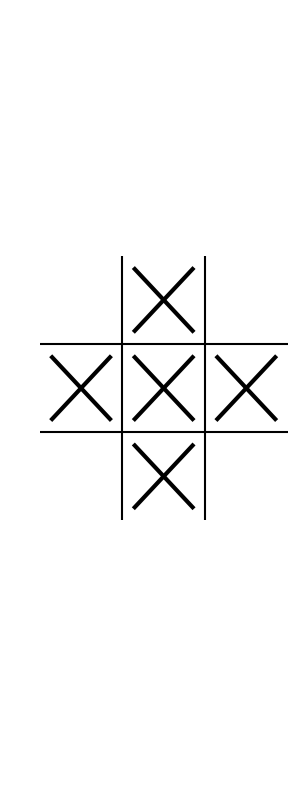}
    \includegraphics[width=0.29\textwidth]{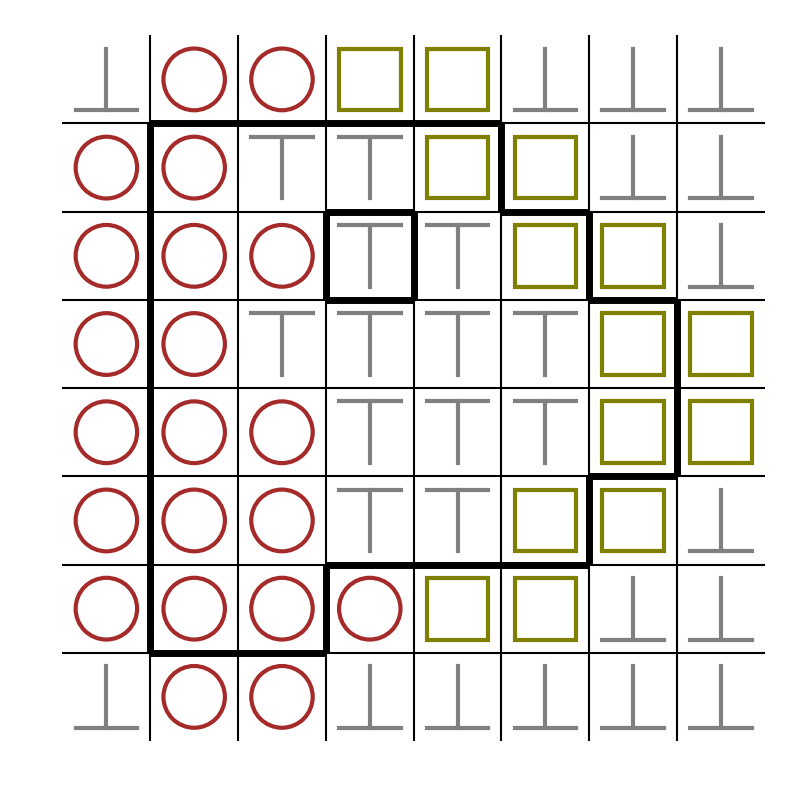}
    \includegraphics[width=0.29\textwidth]{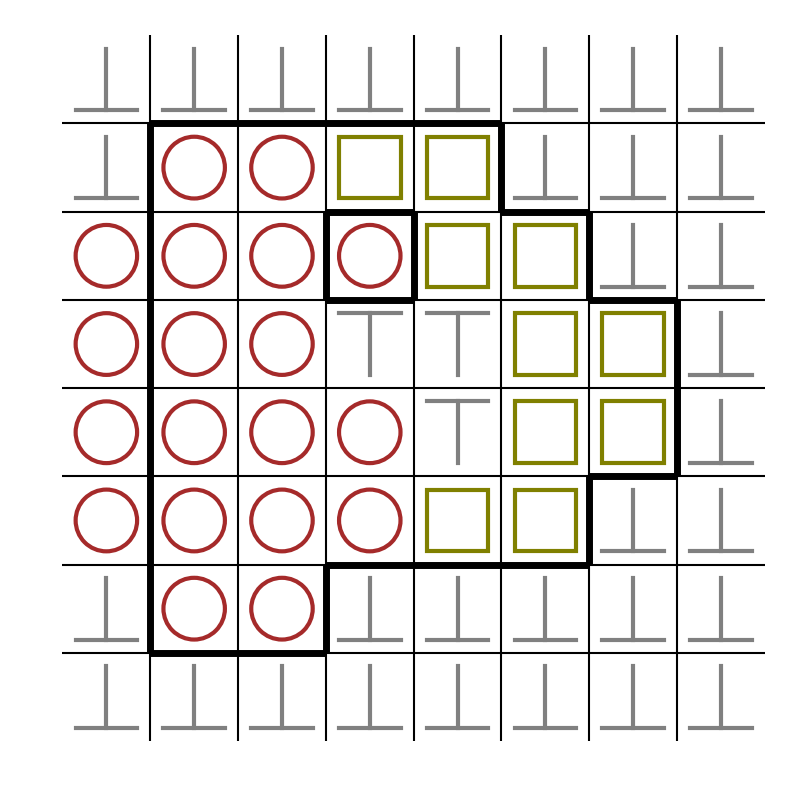}
    \caption{Morphology on label images.}
    \label{fig:label-images}
  \end{subfigure}
 
  \begin{subfigure}{\textwidth}
    \includegraphics[width=0.29\textwidth]{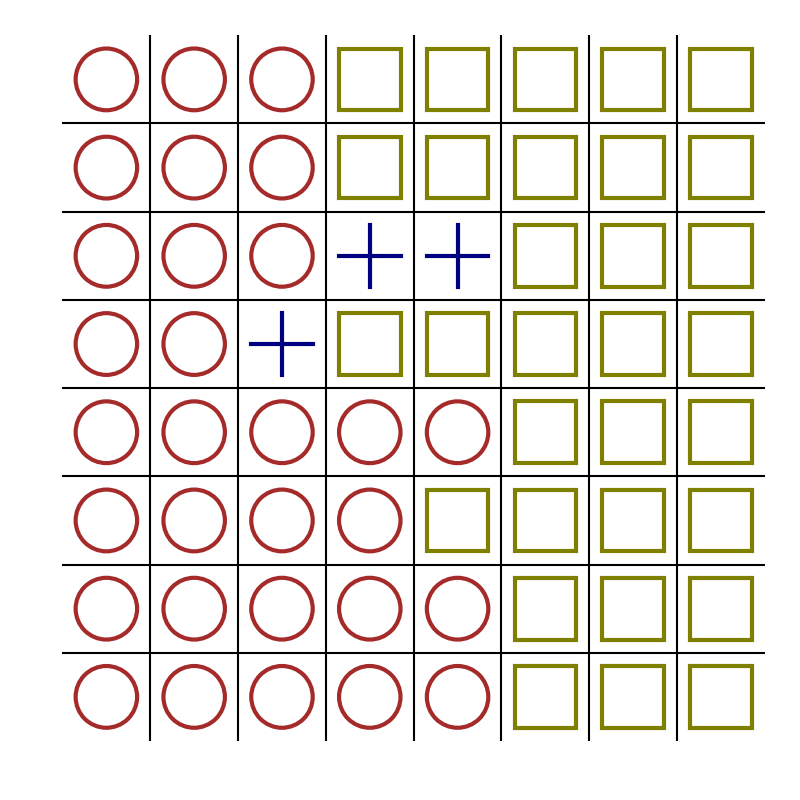}
    \includegraphics[width=0.10\textwidth]{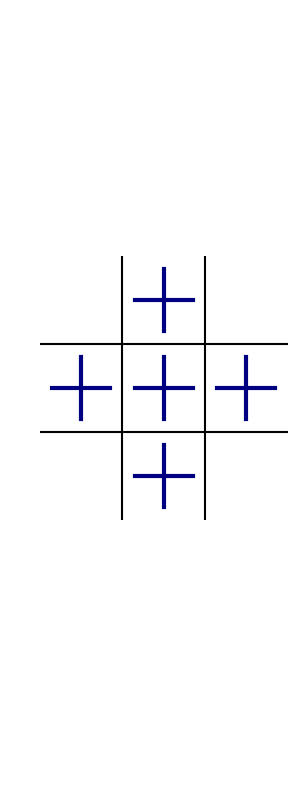}
    \includegraphics[width=0.29\textwidth]{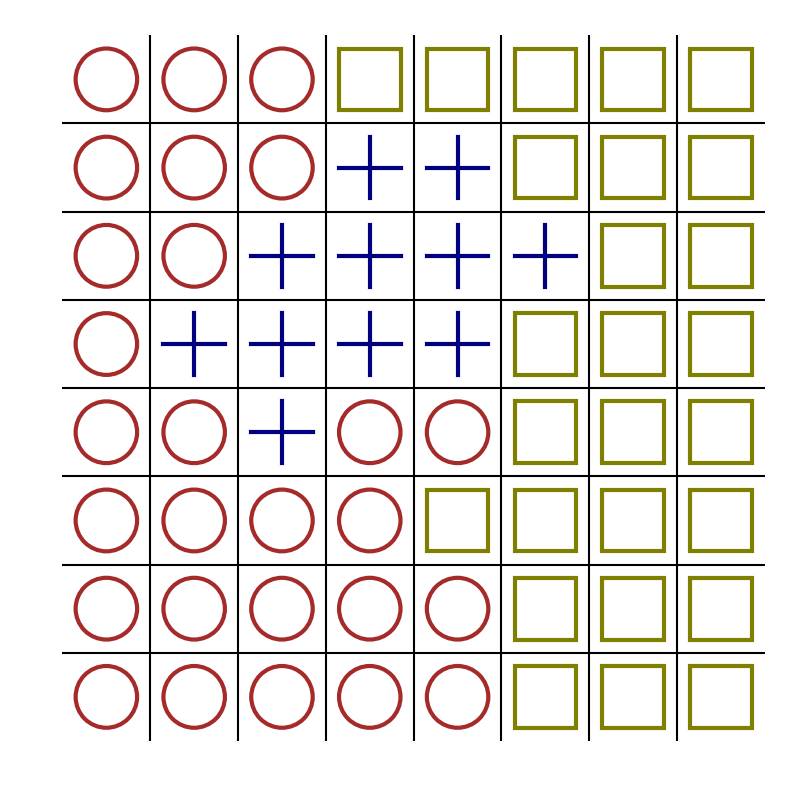}
    \includegraphics[width=0.29\textwidth]{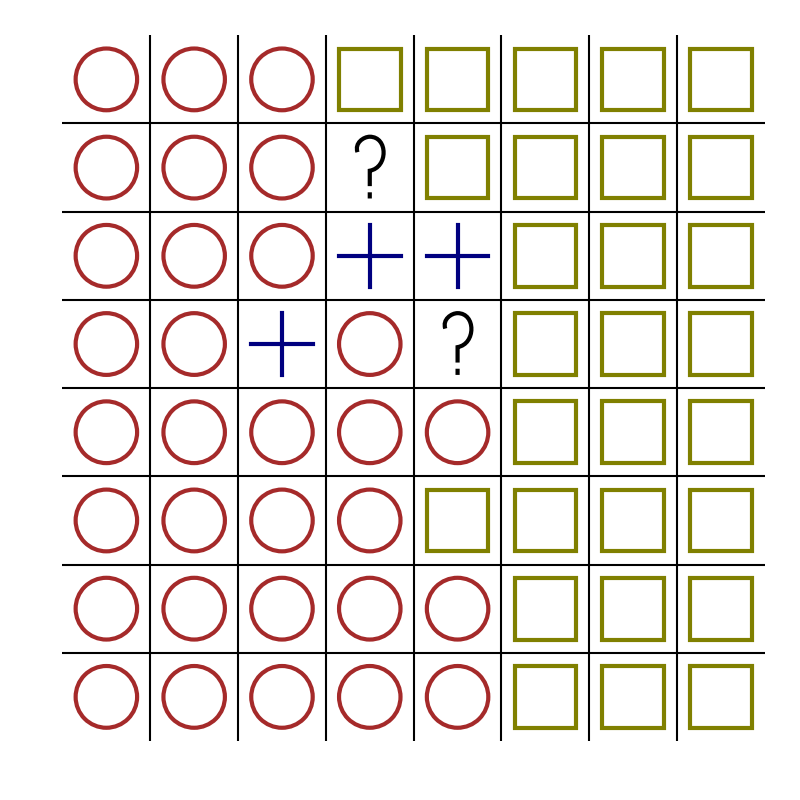}
    \caption{N-ary morphology. The coloring of the structuring element with $plus$ indicates that operations are on the $plus$ category.
      The question mark ? indicate pixels that cannot be assigned a valued without additional ordering of the categories.
      Notice that it is not possible to assign pixels ``no-category'' as in (a), (b) or (c).}
    \label{fig:n-ary}
  \end{subfigure}
  \caption{Comparison of categorical morphologies from the literature. From left to right: $f$, $B$, $\di(f;B)$, $\op(f;B)$.}
  \label{fig:related}
\end{figure}

\subsection{Fuzzy n-ary morphology}
\label{sec:fuzzy-n-ary}
In~\cite{chevallier2016nary} the authors also propose an extension of n-ary morphology to images of categorical distributions. Let $C = \{c_1,c_2,\dots,c_{n+1}\}$ be a set of $n+1$ categories. The categorical distribution of $n+1$ categories is completely determined by a point in the $n$-simplex $\Delta^n = \{ \pi \in \mathbb{R}^{n+1} \mid \pi_k \ge 0, \sum \pi_k = 1\}$, where $\pi_k$ is the probability of $c_k$. An image $f$ is then defined as in \refeq{image} with $\Gamma = \Delta^n$. Operations are again defined on a single category at a time. Let $B_r$ be a closed ball of radius $r$ centered at the origin and let $i$ be the category we operate on. Let $f_k(x) = f(x)_k$ be the probability of observing category $c_k$ in pixel $x$ and let $\omega_k(x) = 1 - f_k(x)$.
Dilation is then defined as
\begin{align}
  \label{eq:fuzzy-n-ary-dilation}
  \di_i(f;B_r)(x)_k =  
  \begin{multipartdef}
    \di(f_k;B_r)(x), & \mycase{k = i},\\
    [1 - \di(f_i;B_r)(x)]\frac{f_k(x)}{\omega_i(x)}, & \mycase{k \ne i}.
  \end{multipartdef}
\end{align}
where $\di(f_i;B_r)(x) = 1 \implies [1 - \di(f_i;B_r)(x)]\frac{f_k(x)}{\omega_i(x)} = 0$.

Two variations on erosion are proposed in \cite{chevallier2016nary}, neither of which we find satisfactory. The first require that we pick a ranking of all categories and does not yield idempotent opening and closing
\begin{align}
  \label{eq:fuzzy-n-ary-erosion-1}
  \er_i(f;B_r)(x)_k &=
  \begin{multipartdef}
    \er(f_k;B_r)(x)                  & \mycase{k = i}\\
    f_k(x) + f_i(x) - \er(f_i;B)(x)  & \mycase k = \min( \arg\min\limits_{j\ne i}(\di(f_j;B)) )\\
    f_k(x)                          & \otherwise
  \end{multipartdef}
\end{align}
The second assumes that the image is restricted to the edges of the simplex (at most two categories are non-zero in any pixel) and opening and closing are again not idempotent
\begin{align}
  \label{eq:fuzzy-n-ary-erosion-2}
  \er_i(f;B_r)(x)_k &=
  \begin{multipartdef}
    \er(f_k;B_r)(x)                                          & \mycase{k = i}\\
    \frac{1 - \er(f_i;B)(x)}{1 - f_i(x)}f_k(x) & \mycase{f_i(x) \le 0.5 \vee \max\limits_{j\ne i}(\di(f_j;B)(x) < 0.5)}\\
    1 - \er(f_i;B)(x)  & \mycase{k = \min(\arg\max\limits_{j\ne i}\di(f_j;B)(x))}\\
    0 & \otherwise
  \end{multipartdef}
\end{align}
We refer the reader to \cite{chevallier2016nary} for the motivation for these formulations and their properties.

\subsection{Fuzzy Pareto morphology}
\label{sec:fuzzy-pareto}
In~\cite{koppen2000pareto} the authors propose fuzzy Pareto morphology for color images. An RGB color image can be seen as a 3-dimensional fuzzy set, where the membership function for each set correspond to the value of each color channel. This can equivalently be seen as point in the half-open unit cube. An image $f$ is then defined as in \refeq{image} with $\Gamma = (0,1]^d$. 
For each $a \in \Gamma$ we can associate a hyper-rectangle defined by the vector from the origin to $a$. Fuzzy Pareto morphology is based on the idea of dominance. For $a,b \in \Gamma$ let $a \cap b = \{\min(a_i,b_i)\}_{i=1\dots d}$ be the intersection of $a$ and $b$. Let $A(a) = \prod_i a_i$ be the area function, yielding the area of the hyperrectangle of $a$. The degree to which $a$ dominates $b$ is then
\begin{equation}
  \label{eq:fuzzy-pareto-dominance}
  \mu_D(a,b) = \frac{A(a \cap b)}{A(b)},
\end{equation}
which measures how much of the hyperrectangle of $b$ is contained in the hyperrectangle of $a$.

\paragraph{}Let $B(x) = \{x+y \mid y \in B\}$, dilation and erosion are then defined as
\begin{align}
  \label{eq:fuzzy-pareto-dilation}
  \di(f;B)(x) &= f\left(\arg\min\limits_{y \in B(x)}\left\{\max\limits_{z \in B(x) \wedge z \ne y}\mu_D(f(z), f(y))\right\}\right),\\
  \label{eq:fuzzy-pareto-erosion}
  \er(f;B)(x) &= f\left(\arg\max\limits_{y \in B(x)}\left\{\min\limits_{z \in B(x) \wedge z \ne y}\mu_D(f(z),f(y))\right\}\right).
\end{align}
Although not directly applicable to categorical distributions, it could easily be extended by either restricting $\Gamma$ to $\{v \in (0,1]^d \mid \sum_i v_i = 1\}$ or by considering it in the context of the Dirichlet distribution. However, \refeq{fuzzy-pareto-dilation} and \refeq{fuzzy-pareto-erosion} are not guaranteed to yield a unique solution, requiring us to come up with an arbitration rule.

\subsection{Morphology on the unit circle}
In~\cite{peters1997mathematical} the authors propose morphology on the unit circle for processing the hue space of color images. The idea is to use structuring elements from the hue space and define an ordering based on the shortest distance along the unit circle between values in the image and values in the structuring element. Although not directly applicable to categorical images, it could be relevant to consider structuring elements that are themselves categorical distributions and base morphology on distance between distributions.

Morphology on the unit circle is also considered in~\cite{hanbury2001morphological} where the authors propose three approaches: using difference operators (e.g. gradient), using grouped data, and using ``labeled openings''. It is the labeled openings that are most relevant in our context. Let $f$ be an image as defined in \refeq{image} with $\Gamma = [0, 2\pi]$. In a labeled opening the unit circle is partitioned into segments $S(\omega) = \{[0,\omega)$, $[\omega, 2\omega)$,$\dots$, $[2\pi-\omega, 2\pi)\}$ and each segment $s \in S(\omega)$ gives rise to a binary image $f(x;s) = f(x) \in s$. A labeled opening is then the union of the binary openings of all segments, $\op_\omega(f) = \cup_{s \in S(\omega)}f(x;s)$, indicating for each pixel if it was opened.
Categorical images have a natural partitioning based on the categories, leading to the set-based morphology in \refsec{color-coded}, where a labeled opening is the pixels that do not change when opened.

\subsection{Morphology on component graphs}
In~\cite{grossiord2019shape} the authors propose a framework for morphology on multi-valued images based on component graphs. Let an image be defined as in \refeq{image} with $\Gamma = \mathbb{R}^d$. The component graph is constructed from the connected components of the level sets of an image. For example, for $d=2$ and $f(x) \in \{0,1\}^2$ the level sets are $\{(0,0), (0,1), (1,0), (1,1)\}$. For each level set we obtain a set of connected components. Each connected component is a node in the component graph and the children of this node are the connected components that are contained in it. In order to construct the component graph it is required that $\Gamma$ allows a minimum, e.g. $\{0\}^d$, such that the graph will be connected. For categorical images this would require that we have a special background category as in \refsec{color-coded} and \refsec{label-images}. Further, it requires that each pixel can have multiple categories, otherwise no component will be nested inside another and the graph will be the root with all connected components as children.

Because the component graph directly exposes the spatial relationship between differently valued regions, it is possible to apply morphological filters, e.g. noise reduction, by pruning some nodes and reconstructing the image from the pruned component graph. Directly pruning the component graph can lead to ambiguity in the reconstruction when a node with two non-comparable parents is removed. The authors propose to solve this by building a component tree of the component graph, prune the tree, then reconstruct the graph from the tree, and the image from the graph. In order to construct the component tree it is necessary to impose a total order on the nodes of the component graph. For example by using a shape measure on the connected components in the component graph.

Because the component graph only captures spatial relationships when connected components overlap for different level sets, some common post-processing operations, such as closing holes in segmentations, are challenging to perform.

\section{Morphology on categorical distributions}
In this section we propose two approaches for morphology on categorical distributions. In \refsec{dirichlet} we show how to operate on all categories simultaneously by operating on Dirichlet distributions. The limitations of this approach will then motivate single category operations that work directly on categorical distributions, which we will define in \refsec{single-category}.

 \subsection{Morphology on Dirichlet distributions}
 \label{sec:dirichlet}
 Let $\mathbb{R}_+$ be the positive real line.  We consider the Dirichlet distribution of order $n \ge 2$ with parameters $\alpha \in \mathbb{R}_+^n$, written as $\Dir(\alpha)$, as a distribution over the $(n-1)$-simplex $\Delta^{n-1} = \{ \pi \in \mathbb{R}^{n} \mid \pi_k \ge 0, \sum \pi_k = 1\}$ with density function

\begin{equation}
  \label{eq:dirichlet-density}
  \mathrm{pdf}(\pi) = \frac{1}{\mathrm{B}(\alpha)}\prod\limits_{k=1}^n \pi_k^{\alpha_k-1}
\end{equation}
where $\mathrm{B}(\cdot)$ is the Beta function. A sample from the Dirichlet distribution of order $n$ can be seen as parameters of the categorical distribution with $n$ categories. Here, we only consider the expectation
\begin{equation}
  \label{eq:dirichlet-projection}
   \mathbb{E}[\mathrm{Dir}(\alpha)_k] = \frac{\alpha_k}{\sum \alpha},
 \end{equation}
 which maps each Dirichlet distribution to a specific categorical distribution. Note that $0 < \alpha_k < \infty$ implies that we can only represent categorical distributions in the open simplex. In practice this is not a problem as we can get arbitrarily close to the boundary of the simplex.
 
 Let $f_k$ be the $k$'th category in $f$.  An image $f$ is defined as in \refeq{image} with $\Gamma = \mathbb{R}_+^n$. If we equip $f$ with the ordering $f \le g \iff [\forall k](f_k(x) \le g_k(x))$ we obtain a complete lattice. Dilation and erosion are then defined as their grayscale counterparts applied to each category independently
\begin{align}
  \di(f;B)(x)_k &= \di(f_k;B),\\
  \er(f;B)(x)_k &= \er(f_k;B).
\end{align}
An example of these operations is provided in \reffig{dirichlet}. Opening and closing are possibly the most interesting operations as they, respectively, decreases and increases uncertainty at the boundaries between overlapping categories.

\begin{figure}[t]
  \centering
  \begin{minipage}[c]{0.09\linewidth}
    $\Dir(\alpha)$
  \end{minipage}
  \begin{minipage}[c]{0.9\linewidth}
    \includegraphics[width=\textwidth,trim=0 20 0 0,clip=true]{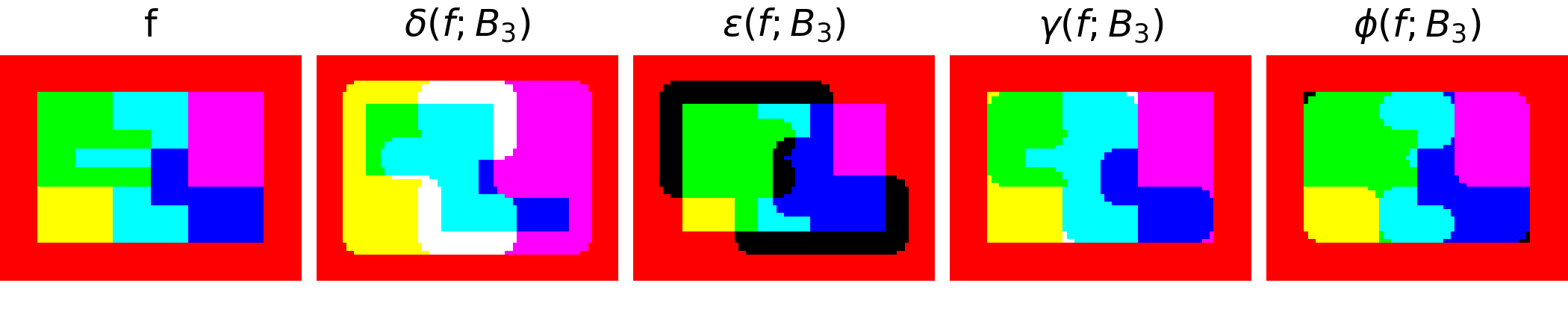}
  \end{minipage}\\
  \begin{minipage}[c]{0.09\linewidth}
    $\E{\Dir}$
  \end{minipage}
  \begin{minipage}[c]{0.9\linewidth}
    \includegraphics[width=\textwidth,trim=0 20 0 20,clip=true]{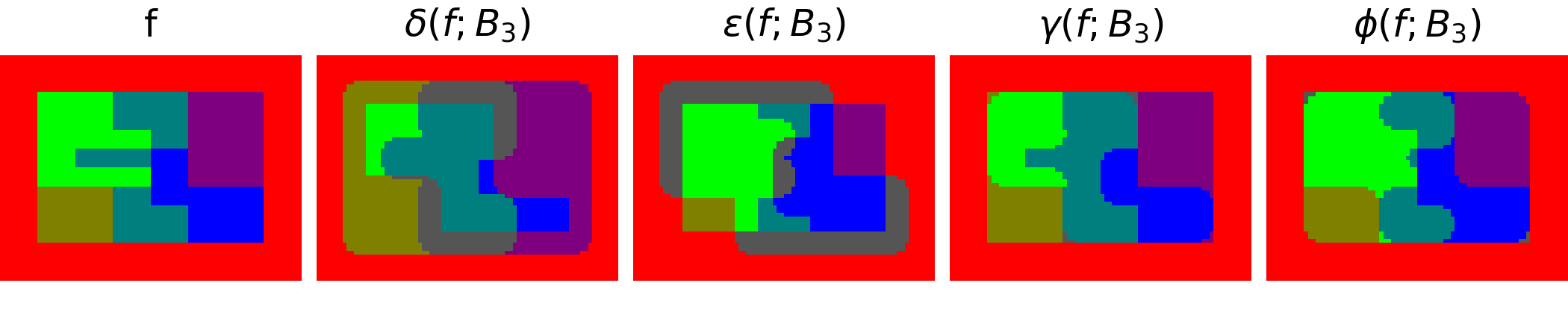}
  \end{minipage}\\
  \begin{minipage}[c]{0.09\linewidth}
    $H(\mathbb{E})$
  \end{minipage}
  \begin{minipage}[c]{0.9\linewidth}
    \includegraphics[width=\textwidth,trim=0 20 0 20,clip=true]{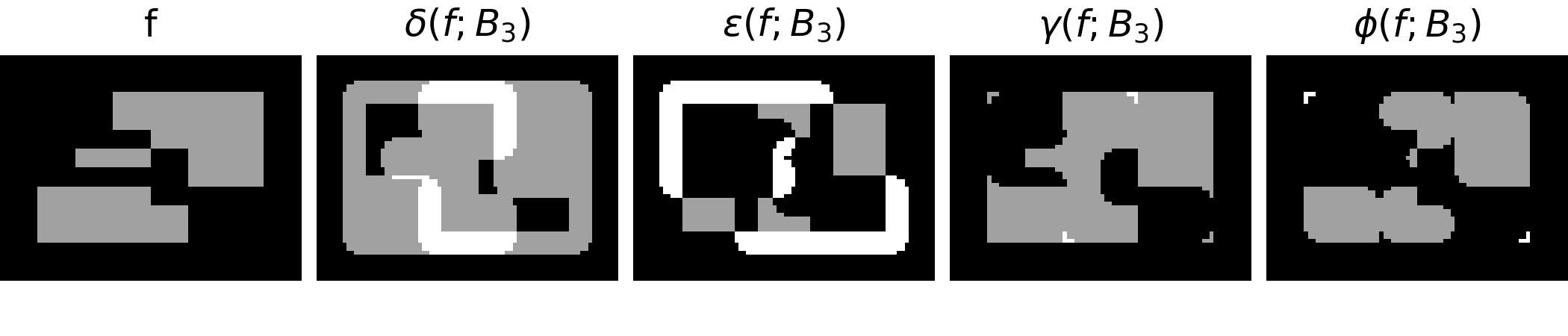}
      \end{minipage}\\
  \begin{minipage}[c]{0.09\linewidth}
    $\|\alpha\|$
  \end{minipage}
  \begin{minipage}[c]{0.9\linewidth}
    \includegraphics[width=\textwidth,trim=0 20 0 20,clip=true]{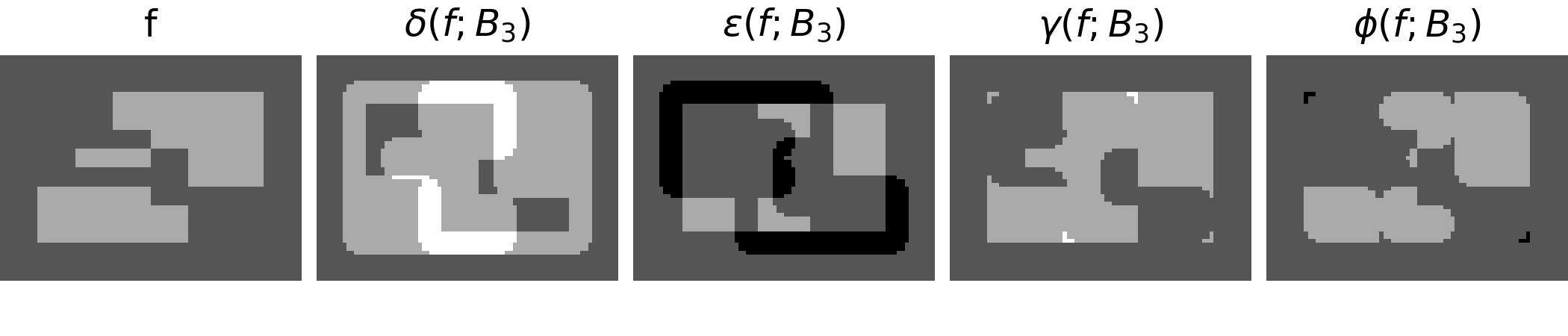}
  \end{minipage}

  \caption[Morphology on Dirichlet distributions]{Morphology on Dirichlet distributions.
    The top-left image is an RGB representation of an image $f$ with three categories, where the colors red, green, and blue correspond to points very close to the vertices of $\Delta^2$ and the remaining colors are mixtures of these three colors.
    The first row is the Dirichlet distribution.
    The second row is the probability vectors obtained from \refeq{dirichlet-projection}.
    The third row is entropy of the probability vectors, and the fourth row is magnitude of the parameter vectors.
    We can see that dilation increases both entropy and magnitude, whereas erosion decreases magnitude and increases or decreases entropy depending on the local distribution.}
  \label{fig:dirichlet}
\end{figure}

We can easily extend these operators to operate on a subset of categories $S$ by simply only updating those categories
\begin{align}
  \di(f;B\vert I)(x)_k &= \begin{multipartdef}
    \di(f_k;B), & \mycase{k \in S}\\
    f_k, & \otherwise
  \end{multipartdef}\\
  \er(f;B\vert I)(x)_k &= \begin{multipartdef}
    \er(f_k;B), & \mycase{k \in S}\\
    f_k, & \otherwise
\end{multipartdef}
\end{align}
An example of these operations is provided in \reffig{dirichlet-subset} where we operate on the green category. Consider the gray/blue region surrounded by green that is indicated with a white ellipse in the left image of the second row. When we dilate the green category we would expect this region to become green in the probability image, but in the Dirichlet space these pixels already have the same green value as the green region, so they are unaffected by the dilation. We could partly solve this by carefully setting the $\alpha$ values, e.g. setting the pixels with only green to have very large green values. However, if our goal is to work on categorical distributions, this becomes too large a burden to be practical and we now turn our attention to morphological operators that work directly on categorical distributions.

\begin{figure}[t]
  \centering
    \begin{minipage}[c]{0.09\linewidth}
    $\Dir(\alpha)$
  \end{minipage}
  \begin{minipage}[c]{0.9\linewidth}
    \includegraphics[width=\textwidth,trim=0 20 0 0,clip=true]{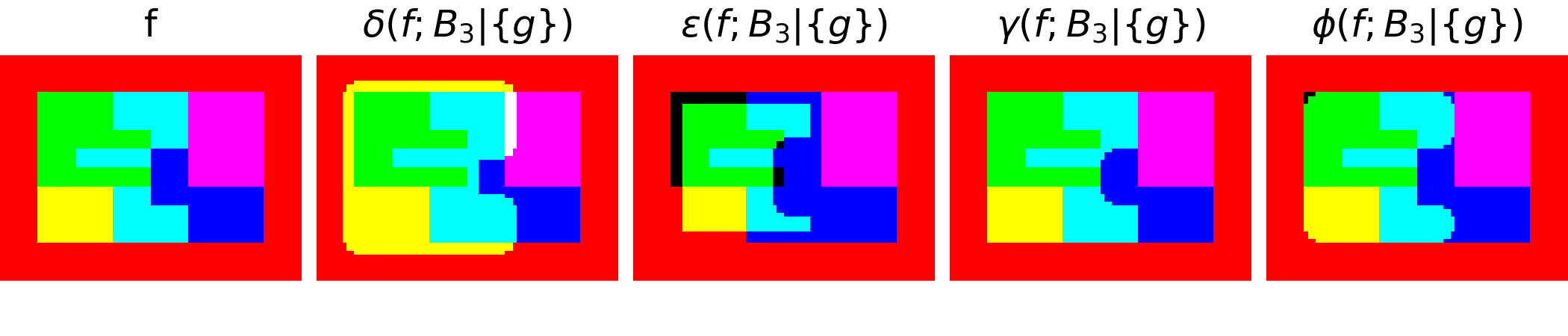}
  \end{minipage}\\
  \begin{minipage}[c]{0.09\linewidth}
    $\E{\Dir}$
  \end{minipage}
  \begin{minipage}[c]{0.9\linewidth}
    \includegraphics[width=\textwidth,trim=0 20 0 20,clip=true]{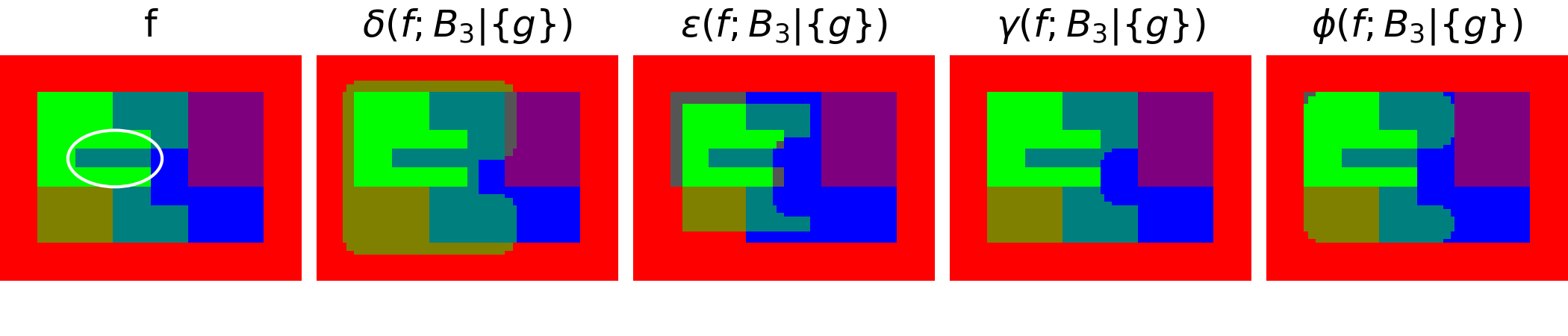}
  \end{minipage}\\
  \begin{minipage}[c]{0.09\linewidth}
    $H(\mathbb{E})$
  \end{minipage}
  \begin{minipage}[c]{0.9\linewidth}
    \includegraphics[width=\textwidth,trim=0 20 0 20,clip=true]{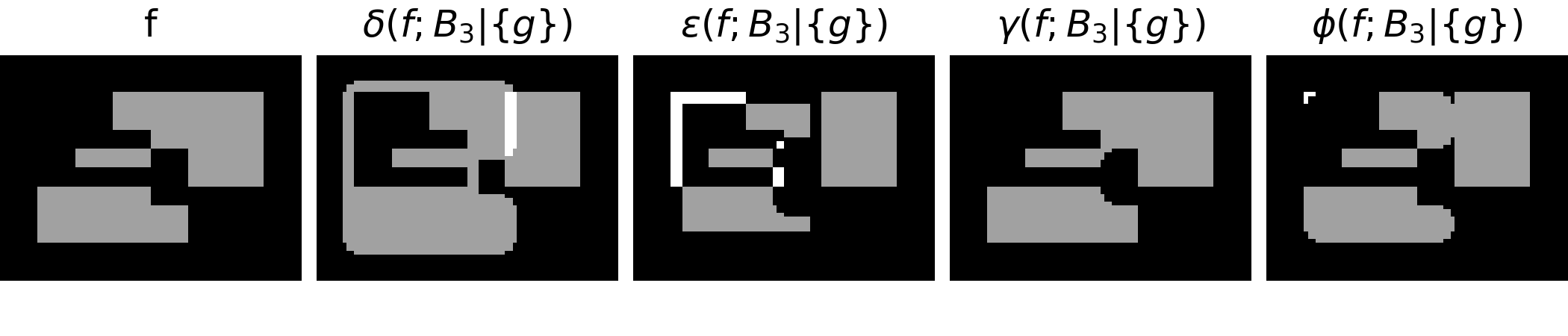}
      \end{minipage}\\
  \begin{minipage}[c]{0.09\linewidth}
    $\|\alpha\|$
  \end{minipage}
  \begin{minipage}[c]{0.9\linewidth}
    \includegraphics[width=\textwidth,trim=0 20 0 20,clip=true]{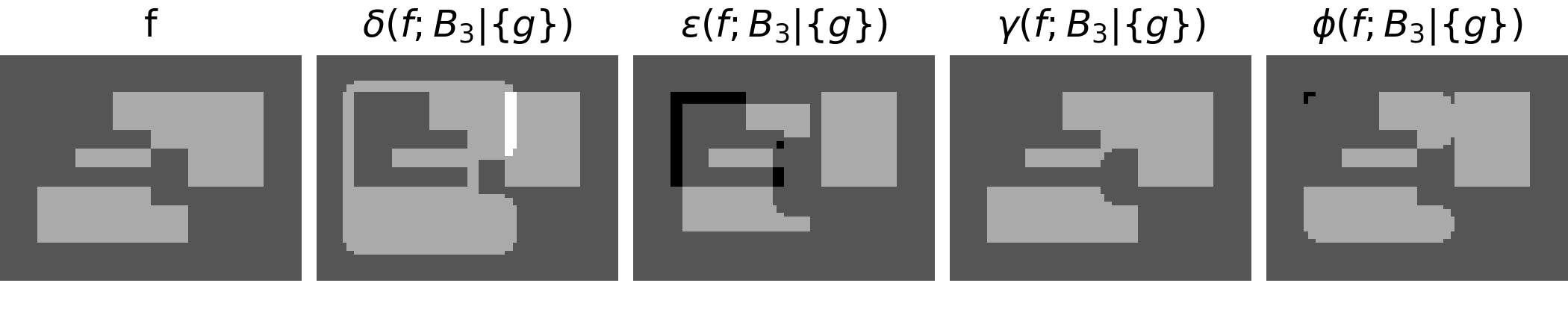}
  \end{minipage}
  
    \caption[Morphology on subset of Dirichlet distribution]{Morphology on Dirichlet distributions using a subset of categories, in this case the green category $\{g\}$. See also \reffig{dirichlet}.}
  \label{fig:dirichlet-subset}
\end{figure}

\subsection{Morphology on categorical distributions}
\label{sec:single-category}
Recall from \refsec{fuzzy-n-ary} that for a set of $n+1$ categories, $C = \{c_1,c_2,\dots,c_{n+1}\}$, the categorical distribution over these categories is completely determined by a point in the $n$-simplex $\Delta^n = \{ \pi \in \mathbb{R}^{n+1} \mid \pi_k \ge 0, \sum \pi_k = 1\}$, where $\pi_k$ is the probability of $c_k$. An image $f$ is then defined as in \refeq{image} with $\Gamma = \Delta^n$. Operations are again defined on a single category at a time. Let $B_r$ be a closed ball of radius $r$ centered at the origin and let $i$ be the category we operate on. Let $f_k(x) = f(x)_k$ be the probability of observing category $c_k$ in pixel $x$ and let $\omega_k(x) = 1 - f_k(x)$.

\subsubsection{Dilation}
\label{sec:dilation}
For the dilated category $i$ the operation is the same as standard grayscale dilation. For the remaining set of categories the operation is a rescaling to ensure that the probabilities sum to one, while the conditional probabilities $\mathrm{Pr}(k \vert x, k \ne i)$ are unchanged
\begin{equation}
  \label{eq:dilation}
  \di_i(f;B_r)(x)_k =  
  \begin{multipartdef}
    \di(f_k;B_r)(x), & \mycase{k = i},\\
    [1 - \di(f_i;B_r)(x)]\frac{f_k(x)}{\omega_i(x)}, & \mycase{k \ne i}.
  \end{multipartdef}
\end{equation}
If $\di(f_i;B_r) = 1$ then the conditional probabilities are not defined and we simply set the probabilities to $1 - \di(f_i;B_r) = 0$. This definition is the same as \refeq{fuzzy-n-ary-dilation} and equivalent to the definition from \cite{chevallier2016nary}.

\subsubsection{Erosion}
\label{sec:erosion}
Erosion is defined similarly to dilation, with the exception of the case when $f_i(x) = 1$ where we cannot rescale the remaining categories because $\omega_i(x) = 0$
\begin{equation}
  \label{eq:erosion}
  \er_i(f;B_r)(x)_k =
  \begin{multipartdef}
    \er(f_k;B_r)(x)                                          & \mycase{k = i}\\
    [1 - \er(f_i;B_r)(x)] \frac{f_k(x)}{\omega_i(x)} & \mycase{k \ne i \wedge f_i(x) < 1}\\
    [1 - \er(f_i;B_r)(x)] \frac{\theta(f_k,B_r)(x)}{\sum_{j \ne i}\theta(f_j,B_r)(x)} & \mycase{k \ne i \wedge f_i(x) = 1}
  \end{multipartdef}
\end{equation}
The function $\theta$ must only depend on the neighborhood defined by $B_r$ and defined such that $\er(f_i;B_r)(x) < 1 \implies [\exists k \ne i]\left(\theta(f_k,B_r)(x) > 0\right)$. In addition we require that, when disregarding discretization issues, eroding with $B_{r+\rho}$ is equivalent to first eroding with $B_{r}$ then eroding with $B_\rho$
\begin{equation}
\er_i(\er_i(f,B_r),B_{\rho})(x) = \er_i(f, B_{r+\rho})(x) \label{eq:erosion-req}.
\end{equation}
Since $\theta$ is only used in the case where $f_i(x) = 1$ we must have that 
\begin{equation}
  \er(f_i;B_r)(x) < 1 \implies 
  \frac{\theta(f_k,B_{r+\rho})(x)}{\sum_{j \ne i} \theta(f_j;B_{r+\rho})(x)} = \frac{\theta(f_k;B_r)(x)}{\sum_{j \ne i} \theta(f_j;B_r)(x)}
\end{equation}
So $\theta$ must only depend on the smallest possible neighborhood $B_{r^*}$ where \\$\er(f_i; B_{r^*}) < 1$, leading to
\begin{align}
  \label{eq:theta}
  \theta(f_k,B_r)(x) &= \di(f_k;B_{r^*})(x)\\
  r^* &= \arg\min\limits_{r' > 0} r',\;  \mathrm{s.t.} \; \er(f_i; B_{r'})(x) < 1. \notag
\end{align}
This amounts to picking the closest category as suggested for crisp categorical images in \cite{chevallier2016nary,vandegronde2017nonscalar}, although without the need for breaking ties since multiple closest categories are now handled by rescaling. In \refapp{proofs} we show that these definitions have the same properties as the definitions in \cite{vandegronde2017nonscalar} for operating on n-ary images.

An example of the proposed operations is provided in \reffig{categorical}, where we operate on the green category. Compared to morphology on Dirichlet distributions using subsets in \reffig{dirichlet-subset} the operations now work directly on the probabilities, making it much easier to understand and control.

\begin{figure}[t]
  \centering
  \includegraphics[width=\textwidth,trim=0 20 0 0,clip=true]{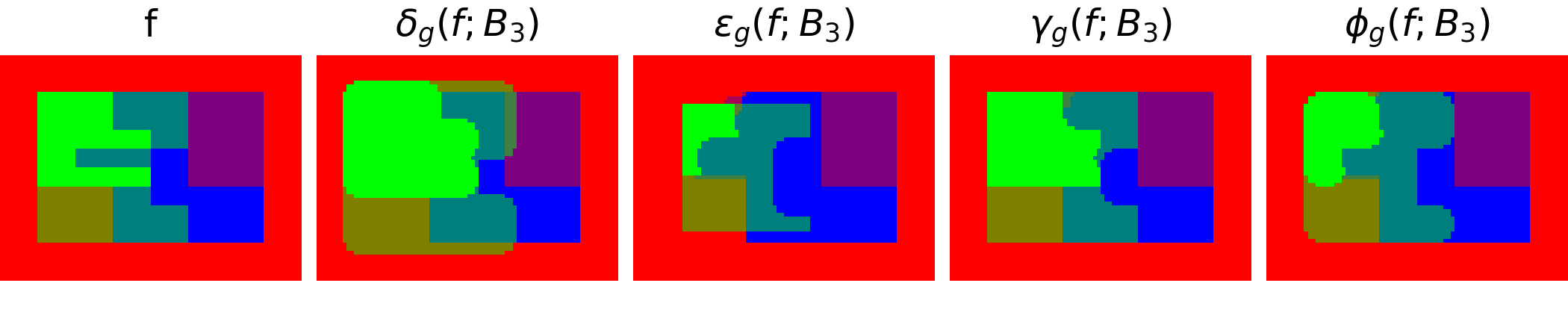}
  \caption{Morphology on categorical distributions. Here we operate on the green category $g$.}
  \label{fig:categorical}
\end{figure}

\section{Protected morphological operations}
\label{sec:protected}
In~\cite{busch1995morphological} the authors introduce the concept of protected morphological operations, where a subset of categories are protected from being updated. Here we adapt the idea of protected morphological operations to categorical distributions and define protected dilation and erosion.

Let $L$ be a set of categories, we then write $\er_i(f;B_r\vert L)$ for an erosion of $i$ that protects $L$. Let $J = C\setminus (\{i\} \cup L)$ be the set of categories that are not protected nor operated on.
Let $f_K(x) = \sum_{k \in K} f_k(x)$ be the sum over a set of categories $K \subset C$. If $L$ is empty, or $[\forall x](f_L(x) = 0)$, protected operations reduce to their non-protected counterparts. 
Because $L$ can change the topology of the domain, we cannot just define operations based on Euclidean distance.
Instead we introduce a distance function $d_\Omega(x,y)$, which computes the distance from $x$ to $y$ on the domain $\Omega$. If $\Omega = \mathbb{Z}^d$, then $d_\Omega(x,y)$ is the Euclidean distance. Computing exact Euclidean distance on a Euclidean domain with holes is non-trivial. Here we use the simplified fast marching method (FMM) from \cite{jones20063d} with the update rule defined in \cite{rickett1999second}, which results in a small approximation error.
For brevity, when possible we leave out function application and write $f$ instead of $f(x)$ in the following.

\begin{figure}[t]
  \centering
  \includegraphics[width=\textwidth,trim=0 20 0 0,clip=true]{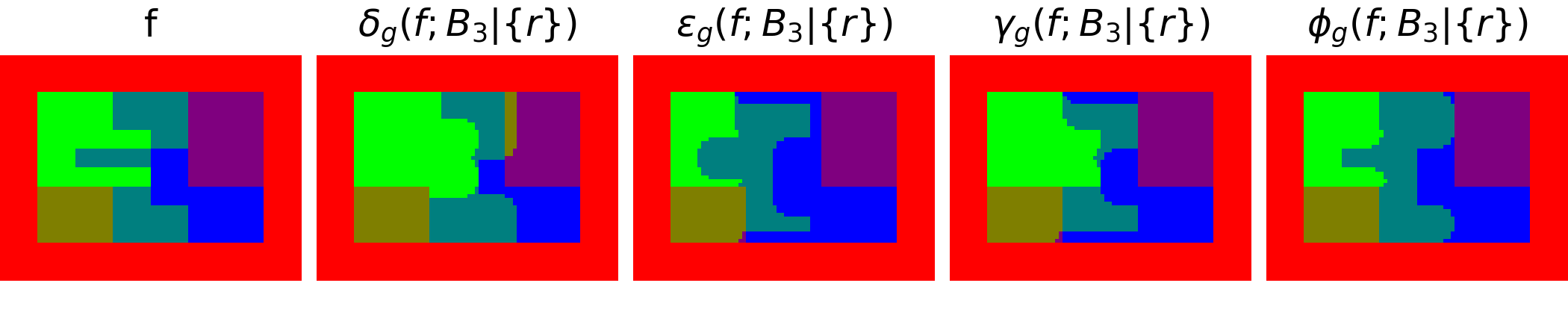}
  \caption{Protected morphology on categorical distributions. The red category $\{r\}$ is protected while we operate on the green category $g$.}
  \label{fig:categorical-protected}
\end{figure}

\subsubsection{Protected dilation}
Let $\Omega_p = \{x \in \D \mid f_L(x) \le 1-p\}$, this is the part of $f$ where it is possible to set $f_i = p$. Protected dilation is then defined as 
\begin{align}
  \label{eq:protected-dilation}
  &\di_i(f;B_r\vert L)(x)_k =  \notag\\
  &\begin{multipartdef}
    f_k                                                  & \mycase{k \in L}\\
    \min\left(1-f_L, \max\limits_{p \in (0,1]}\max\{f_i(y) \mid d_{\Omega_p}(x,y) \le r\}\right)         & \mycase{k = i}\\
    \left[1 - f_L - \di_i(f;B_r\vert L)_i\right]\frac{f_k}{f_J} & \otherwise
  \end{multipartdef}
\end{align}

\subsubsection{Protected erosion}
Protected erosion  is defined similarly to protected dilation, with the added complication of normalization 
\begin{align}
  &\er_i(f;B_r\vert L)(x)_k = \notag\\
  \label{eq:protected-erosion}
  &\begin{multipartdef}
      f_k                                 & \mycase{k \in L}\\
      f_k                                 & \mycase{\max\limits_{p \in (0,1]}\max\{f_J(y) \mid d_{\Omega_p}(x,y) \le r\}} = 0\\
      \min\limits_{p \in (0,1]}\min\{f_i(y) \mid d_{\Omega_p}(x,y) \le r\}     & \mycase{k = i} \\ 
    [1 - f_L -\er_i(f;B_r\vert L)_i]\frac{f_k}{f_J}                                 & \mycase{k \in J \wedge f_J > 0}\\ 
    [1 - f_L -\er_i(f;B_r\vert L)_i]\frac{\theta(f_k)}{\sum_{j\in J}\theta(f_j)}  & \mycase{k \in J \wedge f_J = 0}
  \end{multipartdef}
\end{align}
The first case ensures that all protected categories are unchanged. The second case ensures that a pixel $x$ is not updated, unless there is a path, not blocked by $f_L$, to a pixel $y$ with $f_J(y) > 0$. The importance of this is easily seen by considering the case where $f_i$ varies in an region, but $f_i + f_L = 1$ in the region. The third case states that if there is such a path, then it can be eroded. The fourth and fifth cases handles normalization.
The $\theta$ function is defined in a similar manner as for non-protected erosion in \refeq{theta},
\begin{align}
  \label{eq:protected-theta}
  \theta(f_k)(x) &= \max_{p \in (0,1]}\max\{f_k(y) \mid d_{\Omega_p}(x,y) \le r^*\}\\
  r^* &= \arg\min\limits_{r' > 0} r' \;,\;  \mathrm{s.t.} \; [1 - f_L - \er_i(f;r'\vert L)_i(x)] > 0. \notag
\end{align}
An example of these operations is provided in \reffig{categorical-protected}, where the red category is protected while we operate on the green category. Compared to the non-protected operations in \reffig{categorical} we can see that changes are restricted to the green and blue categories.

\section{Examples}
The first example illustrates how morphology on categorical distributions (\refsec{single-category}) can be used to remove noisy predictions. The second example illustrates how protected morphology on categorical distributions (\refsec{protected}) can be used to model annotator bias.

\subsection{Removing noisy predictions}
Despite the impressive performance of neural networks for segmentation, the results are rarely perfect. \reffig{noisy-predictions} shows part of an electron microscopy image of the hippocampus, along with multi-class predictions and segmentations obtained from \cite{stephensen2020measuring}.  Notice the noisy mitochondria predictions resulting in misclassifications highlighted in \reffig{noisy-predictions-segmentation}.
We can remove these misclassification by opening the mitochondria class before the final classification. \reffig{noisy-predictions-fixed} shows the opened predictions along with the final classifications.
Notice in particular how the errors in circle 2 in \reffig{noisy-predictions-fixed-segmentation} are fixed, such that the vesicle (teal) and the endoplasmic reticulum (yellow) are separated by cytosol. This would have been very difficult to achieve by working directly on the final segmentations. That the vesicle and endoplasmic reticulum are probably misclassified just illustrates that not all things should be fixed in post-processing.

\begin{figure}[t]
  \centering
  \begin{subfigure}[t]{0.32\textwidth}
    \includegraphics[width=\textwidth]{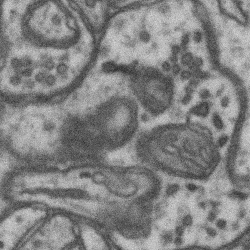}
    \caption{EM image}
  \end{subfigure}
  \begin{subfigure}[t]{0.32\textwidth}
    \includegraphics[width=\textwidth]{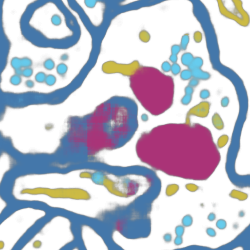}
    \caption{Prediction}
  \end{subfigure}
  \begin{subfigure}[t]{0.32\textwidth}
    \includegraphics[width=\textwidth]{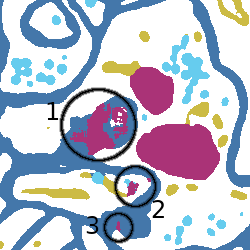}
    \caption{Segmentation}
    \label{fig:noisy-predictions-segmentation}
  \end{subfigure}
  \caption[Noisy mitochondria predictions]{Electron microscopy image of the hippocampus with predictions of five classes: cytosol (white), membrane (blue), mitochondria (purple), endoplasmic reticulum (yellow), and vesicle (teal). By examining neighboring slices, the areas 1-3 have been confirmed to wrongly contain mitochondria predictions.}
  \label{fig:noisy-predictions}
\end{figure}

\begin{figure}[t]
  \centering
    \begin{subfigure}[t]{0.32\textwidth}
  \includegraphics[width=\textwidth]{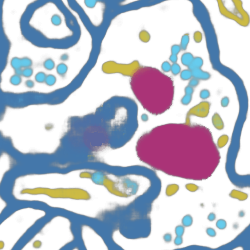}
    \caption{Mitochondria opened}
  \end{subfigure}
  \begin{subfigure}[t]{0.32\textwidth}
    \includegraphics[width=\textwidth]{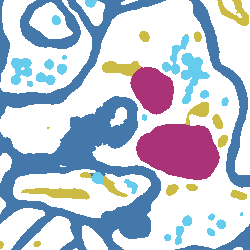}
    \caption{Segmentation of (a)}
  \end{subfigure}
  \begin{subfigure}[t]{0.32\textwidth}
    \includegraphics[width=\textwidth]{remove-noisy-predictions/246_0-1536_0-2048_cutout_segmentation_errors-highlighted.png}
    \caption{Original segmentation}
    \label{fig:noisy-predictions-fixed-segmentation}
  \end{subfigure}
  \caption[Fixed mitochondria predictions]{Fixing mitochondria misclassifications by opening the mitochondria predictions with $B_{12}$. }
  \label{fig:noisy-predictions-fixed}
\end{figure}

\subsection{Modeling annotator bias}
Expert annotation is the gold standard in most clinical practice as well as for evaluating computer methods. However, annotation tasks are inherently subjective and prone to substantial inter-rater variation \cite{joskowicz2019inter,becker2019variability}. When investigating the influence of this variation on statistics and decisions it can be interesting to consider specific hypotheses regarding the variation. Consider the brain tumor annotation in \reffig{braintumor}. The annotation is derived from the QUBIQ\footnote{\url{https://qubiq.grand-challenge.org/}} challenge brain tumor dataset, where three annotators each annotated whole tumor, tumor core and active tumor. From this we obtain an image with four categories: background, edema, active core, inactive core. Although the annotators have a high level of agreement, there is still substantial variation in the extent of edema and in how much of the tumor core is active.

\begin{figure}[t]
  \centering
  \includegraphics[width=0.49\textwidth]{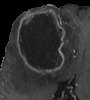}
  \includegraphics[width=0.49\textwidth]{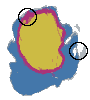}
  \caption[Brain tumor annotation]{Inter-rater variation in annotation of brain tumors. White is background, blue edema, yellow inactive core and purple active core. Variation is indicated by color mixing. The black circles highlights two regions with large variation. }
  \label{fig:braintumor}
\end{figure}

Using protected dilation we can for example hypothesize how the merged annotation would appear under the assumption that the tumor core is oversegmented but the active part is undersegmented. \reffig{active-core} shows the results where we first dilate the active core while protecting edema and background, then dilate edema while protecting background. This would allow us to easily investigate if statistical differences in a case-control study could be explained by biased annotations.

\begin{figure}[t]
  \centering
  \begin{subfigure}[t]{0.24\textwidth}
    \includegraphics[width=\textwidth]{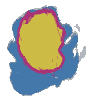}
    \caption{Original}
  \end{subfigure}
  \begin{subfigure}[t]{0.24\textwidth}
    \includegraphics[width=\textwidth]{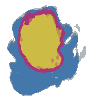}
    \caption{$B_1$}
  \end{subfigure}
  \begin{subfigure}[t]{0.24\textwidth}
    \includegraphics[width=\textwidth]{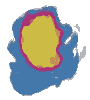}
    \caption{$B_2$}
  \end{subfigure}
  \begin{subfigure}[t]{0.24\textwidth}
    \includegraphics[width=\textwidth]{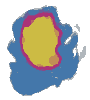}
    \caption{$B_3$}
  \end{subfigure}
  \caption{{\it What could the annotation look like if the core was oversegmented, but the active part undersegmented?} Dilation of active core while protecting edema and background, followed by dilation of edema while protecting background using $B_1, B_2, B_3$.}
  \label{fig:active-core}
\end{figure}

\section{Discussion \& Conclusion}
We have provided a thorough review of morphology on categorically valued images. Based on this we have defined morphology on Dirichlet distributions and morphology on categorical distributions. Inspired by \cite{busch1995morphological} we have further defined protected morphology on categorical distributions. We have demonstrated the behavior of the proposed operations and shown how they can be used in real-world applications such as noise removal in multi-class predictions and modeling annotator bias.

The definition of dilation is straightforward and no obvious alternatives present themselves. This is not so for erosion. In our definition, erosion corresponds to conditioning on a change in probability of the eroded category. An equally valid approach would be to also condition on where this change came from. Instead of simply rescaling the categories with non-zero mass we could include information from the neighborhood. For example, when eroding $i$ we would fill the difference $f_i(x) - \er(f_i;B_r)(x)$ based on the pixels that contribute to the difference, that is, those with minimum mass for $i$.
This would result in smoother boundaries, which could be a better representation of uncertainty. A downside is that categories can leak into each other, leading to undesirable results. 

In this work we have focused on the basic morphological operations, dilation and erosion, and their compositions, closing and opening. A logical next step is to investigate more complex morphological operations, such as the morphological gradient, which may be used to investigate spatial relationship between categories by measuring the change in one category as a function of change in another category.

We have defined protected versions of dilation and erosion. From these we could define opening and closing in the standard way. Alternatively, by changing which categories are protected for dilation and erosion we get more control over how a category is opened or closed. In~\cite{busch1995morphological} the authors explore similar ideas for so called ``tunneling'' and ``bridging'' operations on their set-based morphology, which would be interesting to consider in the context of categorical distributions.

Our aim in this work was to bring morphological operations to probabilistic representations of categorical images. These representations can be considered as generative processes that can be sampled. Naive sampling will result in noisy and unrealistic samples. Combining the sampling process with the proposed morphological operations could be an easy approach to obtain smoother and more realistic samples.

In summary, morphology is an indispensable tool for post-processing segmentations. Extending morphology to categorical images and their probabilistic counterparts presents a particular problem since there is in no inherent ordering of categories. In this paper, we have proposed to view categorical images as images of categorical distributions and defined morphological operations that are consistent with this view.

\bibliography{bibliography}

\begin{appendices}
\section{Proofs for \refsec{single-category}}
\label{app:proofs}
We closely follow Section~5.2 in \cite{vandegronde2017nonscalar} by defining a preorder $\le_i$ on $\mathcal{F}$ and showing that our definitions of dilation and erosion form an adjunction in this preorder. We then show that their compositions are an opening ($\op_i = \di_i\er_i$) and a closing ($\cl_i = \er_i \di_i$), where we define an opening as an operator that is increasing, anti-extensive and and idempotent, and a closing as an operator that is increasing, extensive and idempotent
\begin{align}
  &\subeq{a}\; f \le_i g \implies \op_if \le_i \op_i g,  && \subeq{b}\; \op_i f \le_i f, && \subeq{c}\; \op_i \op_i f = \op_i f\\
  &\subeq{a}\; f \le_i g \implies \cl_if \le_i \cl_i g,  && \subeq{b}\; f \le_i \cl_i f, && \subeq{c}\; \cl_i \cl_i f = \cl_i f
\end{align}

\paragraph{}For two images $f,g \in \mathcal{F}$ we define the preorder $\le_i$ as
\begin{equation}
  f \le_i g \iff [\forall x](f_i(x) \le g_i(x))\label{eq:preorder}
\end{equation}
This preorder is not antisymmetric, as we can have $f \le_i g$ and $g \le_i f$ but not $f = g$.

\begin{theorem}
  \label{thm:adjunction}
  $\di_i$ and $\er_i$ form an adjunction in the preorder $\le_i$
  \begin{equation}
    \label{eq:adjunction}
    \di_i(f;B_r) \le_i g \iff f \le_i \er_i(g;B_r)
  \end{equation}
\end{theorem}

\begin{proof}
Since categories $j \ne i$ have no influence on $\le_i$ we only need to consider the case where $k = i$ in \refeq{dilation} and \refeq{erosion}. These cases are standard grayscale dilation and erosion that form an adjunction.
\end{proof}

\begin{lemma}
  $\di_i$ and $\er_i$ are increasing in $\le_i$
\end{lemma}
\begin{proof}
  Follows by the same argument as for \refthm{adjunction}.
\end{proof}

\begin{theorem}
  For a fixed structuring element $B_r$, $\di_i$ and $\er_i$ satisfy
  \begin{align}
                     f &\le_i \er_i(\di_i(f)) \label{eq:f-edf}\\ 
        \di_i(\er_i(f)) &\le_i f              \label{eq:def-f}\\
    \di_i(\er_i(\di_i(f))) &= \di_i(f)        \label{eq:dedf-df}
  \end{align}
\end{theorem}
\begin{proof}
  \refeq{f-edf} and \refeq{def-f} follows from substitution into \refeq{adjunction} with $g = \di_i(f)$ and $f = \er_i(g)$
  \begin{alignat}{6}
    \delta_i(f;B_r)        &\le_i \di_i(f;B_r) &\; \Longrightarrow \; & &         f & \le_i \er_i(\di_i(g;B_r);B_r)\\
    \di_i(\er_i(g;B_r);B_r) &\le_i g            &\; \Longleftarrow  \; & & \er_i(g;B_r) & \le_i \er_i(g;B_r)
  \end{alignat}

  To show \refeq{dedf-df} we consider the three cases (I) $k = i$, (II) $k \ne i$ with $[\forall x](f_i(x) < 1)$, and (III) $k \ne i$ with $[\exists x](f_i(x) = 1)$. For brevity, we leave out the structuring element $B_r$, pixel index $x$ and parentheses from operator and function application in the following.
  
  \paragraph{}For (I) $k = i$ we can directly substitute the definitions in \refeq{dilation} and \refeq{erosion} to get
  \begin{align}
    (\di_i\er_i\di_if)_i &= \di\er\di f_i \\
                        &= \di f_i \\
                        &= (\di_if)_i,
  \end{align}
  where the second step follows from the properties of the standard grayscale operations.

  \paragraph{}For (II) $k \ne i$ with $[\forall x](f_i(x) < 1)$ we can ignore the third case in \refeq{erosion}. After substitution and cancellation of terms we get,
  \begin{align}
    (\di_i\er_i\di_i f)_k
      &=  \left[1 - \di f_i\right] \frac{f_k}{f_J}
        = (\di_i f)_k
  \end{align}

  \paragraph{}For (III) $k \ne i$ with $[\exists x](f_i(x) = 1)$, the third case in \refeq{erosion} is only relevant when $(\di_i f)_i(x) = (\di f_i)(x) = 1$, which leads to both sides of \refeq{dedf-df} being zero
\begin{align}
  (\di_i\er_i\di_if)_k &= [1 - \di f_i] \frac{\theta(\er_i \di_i f)_k}{\sum_{j\in J} \theta(\er_i\di_if)_j} = 0\\
  (\di_i f)_k &= [1 - \di f_i] \frac{f_k}{f_J} = 0
\end{align}
So $\di_i\er_i\di_if = \di_if$
\end{proof}

\begin{corollary}
  $\op_i = \di_i \er_i$ is an opening and $\cl_i = \er_i \di_i$ is a closing
\end{corollary}
\begin{proof}
  Corollary 1. in \cite{vandegronde2017nonscalar}
\end{proof}
\end{appendices}

\end{document}